\documentclass[runningheads]{llncs}

\usepackage{amsthm}
\newtheorem{thm}{Theorem}
\newtheorem{lem}{Lemma}

\usepackage{url}

\usepackage{xcolor}
\usepackage{amsmath}
\usepackage{graphicx}
\graphicspath{{figures/}}
\usepackage[pdf]{graphviz}
\usepackage{tikz}

\usepackage{adjustbox}
\usepackage{array}
\newcolumntype{R}[2]{%
    >{\adjustbox{angle=#1,lap=\width-(#2)}\bgroup}%
    l%
    <{\egroup}%
}
\newcommand*\rot{\multicolumn{1}{R{45}{1em}}}

\usepackage{algorithm}
\usepackage{algorithmicx}
\usepackage{algpseudocode}
\algblock{Input}{EndInput}
\algnotext{EndInput}
\algblock{Output}{EndOutput}
\algnotext{EndOutput}
\newcommand{\Desc}[2]{\State \makebox[2em][l]{#1}#2}

\let\llncssubparagraph\subparagraph
\let\subparagraph\paragraph
\usepackage{titlesec}
\let\subparagraph\llncssubparagraph

\titlespacing{\section}{4pt}{3pt plus 4pt minus 2pt}{2pt plus 2pt minus 2pt}
\titlespacing{\subsection}{3pt}{2pt plus 4pt minus 2pt}{2pt plus 2pt minus 2pt}
\newcommand{\tool}{Proof Blocks}

\newcommand{\bigO}{\ensuremath{\mathcal{O}}}

\begin{document}
\title{Efficient Feedback  and Partial Credit Grading for \tool{} Problems}
%
%

\author{Seth Poulsen\orcidID{0000-0001-6284-9972}
\and
Shubhang Kulkarni\orcidID{0000-0002-1670-6011}
\and
Geoffrey Herman\orcidID{0000-0002-9501-2295}
\and
Matthew West\orcidID{0000-0002-7605-0050}
}
\authorrunning{S. Poulsen et al.}
\institute{University of Illinois Urbana-Champaign, Urbana, IL 61801, USA
\email{sethp3@illinois.edu}\\
}

\maketitle              

\begin{abstract}
\vspace{-2.5em}
\tool{} is a software tool that allows students to practice writing mathematical
proofs by dragging and dropping lines instead of writing proofs from scratch.
\tool{} offers the capability of assigning partial credit and providing
solution quality feedback to students. This is done by computing the
\emph{edit distance} from a student's submission to some predefined set of
\emph{solutions}.
In this work,
we propose an algorithm for the edit distance problem that significantly
outperforms the baseline procedure of exhaustively
enumerating over the entire search space. Our algorithm relies on a reduction to
the minimum vertex cover problem.
We benchmark our algorithm on thousands of student submissions from multiple
courses, showing that the baseline algorithm is
intractable, and that our proposed
algorithm is critical to enable classroom deployment.
Our new algorithm has also been used for problems in many other domains where
the solution space can be modeled as a DAG, including but not limited to
Parsons Problems for writing code, helping students understand packet ordering
in networking protocols, and helping students sketch solution steps for physics
problems. Integrated into multiple learning management systems, the algorithm
serves thousands of students each year.
\vspace{-1em}
\end{abstract}

\vspace{-2em}
\section{Introduction}
Traditionally, classes that cover mathematical proofs expect students
to read proofs in a book, watch their instructor write proofs, and then write
proofs on their own. Students often find it difficult to jump to writing proofs on
their own, even when they have the required content
knowledge~\cite{weber2001student}.
Additionally, because proofs need to be graded manually, it often takes a while for
students to receive feedback on their work.

\tool{} is a software tool that allows students to construct a mathematical proof
by dragging and dropping instructor-provided lines of a proof instead of writing
from scratch (similar to Parsons
Problems~\cite{parsons2006parson} for writing code---see Figure~\ref{fig:example1}
for an example of the \tool{} user interface).
This tool scaffolds students' learning as they transition from reading proofs to
writing proofs while also providing instant machine-graded feedback.
To write a \tool{} problem, an instructor specifies lines of a proof and their
logical dependencies. The autograder accepts any ordering of the lines that
satisfies the dependencies.

\begin{figure*}
\centering
\resizebox{0.8\textwidth}{!}{%
\begin{tikzpicture}
\node[anchor=north west,inner sep=0] at (0,0)
{\includegraphics[width=\textwidth]{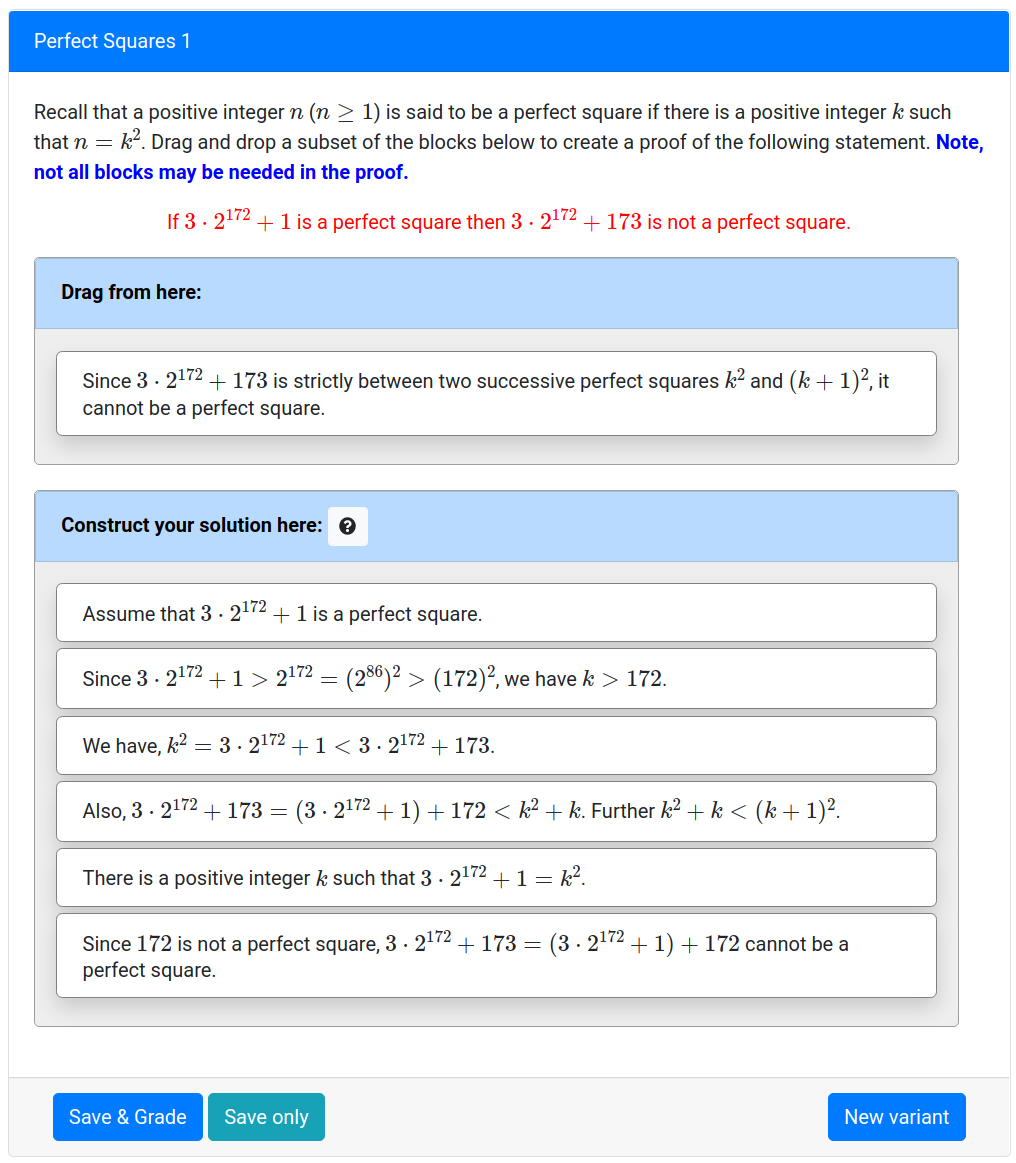}};
\node[shape=circle,draw=black,scale=1.2] (1) at (-0.25,-7.37) {1};
\node[shape=circle,draw=black,scale=1.2] (2) at (-0.25,-10.6) {2};
\node[shape=circle,draw=black,scale=1.2] (3) at (-0.25,-8.26) {3};
\node[shape=circle,draw=black,scale=1.2] (4) at (-0.25,-9.0) {4};
\node[shape=circle,draw=black,scale=1.2] (5) at (-0.25,-9.8) {5};
\node[shape=circle,draw=black,scale=1.2] (6) at (-0.25,-4.71) {6};
\node[shape=circle,draw=black,scale=1.2] (7) at (-0.25,-11.4) {7};
\end{tikzpicture}
}
 \hfill
\raisebox{0.6\height}{
\digraph[scale=0.42]{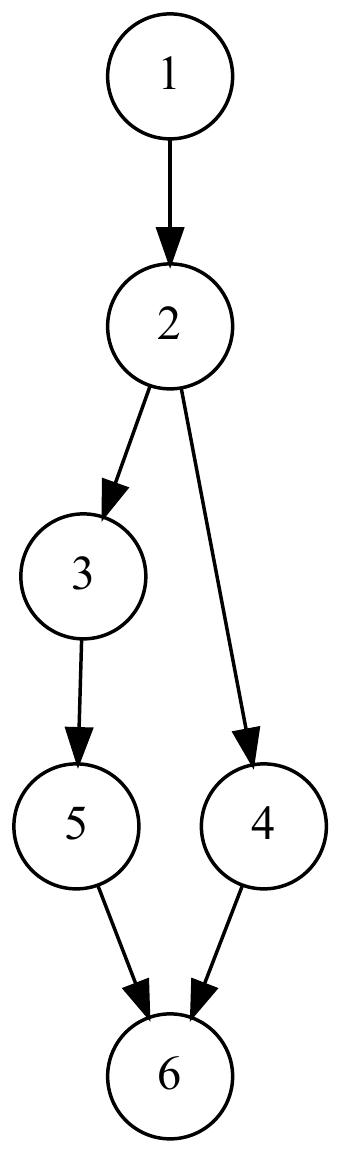}{
margin=0;
node [shape="circle"];
1 -> 2
2 -> 3
2 -> 4 -> 6
3 -> 5 -> 6
}
 }
\caption{
An example of the \tool{} student-user interface.
The instructor wrote the
problem with 1, 2, 3, 4, 5, 6 as the intended solution, but the \tool{} autograder
will also accept any other correct solution as determined by the dependency graph
shown. For example, both 1, 2, 4, 3, 5, 6 and 1, 2, 3, 5, 4, 6
would also be accepted as correct solutions.
Line 7 is a distractor that does not occur in any correct
solution.
\vspace{-1em}
}
\label{fig:example1}
\end{figure*}

Calculating the least edit distance from a student submission to some
correct solution solves two problems: (1) assigning students partial
credit (a key concern for students taking exams in a computerized
environment~\cite{apostolou2009student,darrah2010comparative}) and (2) giving
students instant feedback on their work, which can be a huge help for
students~\cite{anderson1995cognitive}. The baseline algorithm performance is not
sufficient to scale to provide immediate feedback to students in large classrooms,
necessitating a more efficient solution.
This paper makes the following contributions:
\begin{itemize}
\item An efficient algorithm for calculating the minimum edit distance from a
sequence to some topological ordering of a directed acyclic graph (DAG)
\item Application of this algorithm to grading and feedback for \tool{} problems
\item Mathematical proofs that the algorithm is correct, and has asymptotic
complexity that allows it to scale to
grading many student submissions at once, even for large problems
\item Benchmarking results on thousands of student submissions
showing that the efficient algorithm is needed to enabling the performance
necessary for classroom deployment
\end{itemize}

\section{Related Work}

\subsection{Software for Learning Mathematical Proofs}
Work in intelligent tutors for mathematical proofs goes back to work by John
Anderson and his colleagues on The Geometry Tutor
\cite{anderson1985geometry,anderson1995cognitive,koedinger1990abstract}. More
recently, researchers have created tutors for propositional logic, most notably Deep
Thought~\cite{mostafavi2015data,mostafavi2017evolution} and
LogEx~\cite{lodder2020providing-1,lodder2019comparison}.
The authors' prior work reviews other
software tools that provide visual user interfaces for constructing
proofs~\cite{poulsen2022proof}.

Most of these tools cover only a small subset of the material typically covered in
a discrete mathematics course, for example, only propositional logic. Those tools
that
are more flexible require learning complex theorem prover languages.
In contrast, \tool{} enables instructors to easily provide students with proof
questions on any topic by abstracting the content from the grading mechanism. The
downside of this is that \tool{} is not able to give students content-specific
hints as some of the other tools are.

\subsection{Edit Distance Based Grading and Feedback}
To our knowledge, no one has ever used edit distance based grading as a way of
providing feedback for mathematical proofs, but edit distance based grading
algorithms have been used in other contexts.

Chandra et al.~\cite{chandra2019automated} use edit distance to assign partial
credit to incorrect SQL queries submitted by students, using reference solutions
provided by the instructor.
Edit distance based methods, often backed by a database of known correct solutions,
have also been used to give feedback to students learning to program in general
purpose programming languages~\cite{paassen2018continuous,gulwani2018automated}
One difference between these and our method is that in programming contexts, the
solution space is very large, and so the methods work based on edit distance to some
known correct solution (manually provided by the instructor or other students).
Because we model mathematical proofs as DAGs, we are able to constrain the solution
space to be small enough that our algorithm can feasibly check the shortest edit to
\emph{any} correct solution.

Alur et al.~\cite{alur2013automated} provide a framework for automatically grading
problems where students must construct a deterministic finite automata (DFA). They
use edit distance for multiple purposes in their multi-modal grading scheme.





\section{Proof Blocks}
 Prior work has shown that Proof Blocks are
effective test questions, providing about as much information about student
knowledge as written proofs do~\cite{poulsen2021evaluating}, and also show promise
in helping students save time when learning to write
proofs~\cite{poulsen2023efficiency}.
To write a \tool{} problem, an instructor provides the proof lines and
the logical dependencies between the lines of the proof. These logical dependencies
form a DAG.
The autograder gives the
student points if their submission is a topological sort of the dependency graph.
On exams or during in-class activities, students are often given multiple tries to
solve a problem, so it is critical
that they receive their feedback quickly.
Additional details about the
instructor and student user interfaces, as well as best practices for using \tool{}
questions are given in a tool paper~\cite{poulsen2022proof}.

\tool{} is currently integrated
into both PrairieLearn~\cite{prairielearn} and Runestone
Interactive~\cite{miller2012beyond}. In PrairieLearn, students who submit their
work are shown their score and told the first line of their proof that isn't
logically supported by previous lines.
The Runestone implementation highlights certain lines of the student submission
that should be moved in order to get to a correct solution.
Research and discussion about which types of feedback are most helpful for student
learning are of crucial importance, but are beyond the scope of this paper, which
will focus on the technical details of the edit distance algorithm which enables
the construction of feedback.

Our algorithm assumes that each block is of equal weight for assigning partial
credit. The benefit of this is that the algorithm can assign partial credit for
any \tool{} problem without needing to know anything about the content of the
blocks, making it quicker and easier for instructors to write questions.
We have also had instructors using \tool{} express an interest in having the
ability to have some blocks weighted more than others in the grading. We leave
this to future work.

\section{The Edit Distance Algorithm}
Before defining our grading algorithms rigorously, it will first help to set up
some formalism about \tool{} problems. We then give the baseline version of
the algorithm, followed by an optimized version which is necessary for production
deployment.  For simplicity, our focus in communicating the algorithms will be on
calculating the edit distance, but the
proof of the correctness of the algorithm also explicitly constructs the edit
sequence used to give students feedback on how
to edit their submission into a correct solution.
\subsection{Mathematical Preliminaries}

\paragraph{Graph Theory.} Let $G = (V, E)$ be a DAG. Then a subset of vertices
$C\subseteq V$ is a \emph{vertex
cover} if every edge in $E$ is incident to some vertex in $C$. The \emph{minimum
vertex cover} (MVC) problem is the task of finding a vertex cover of minimum
cardinality. In defining our algorithms, we will assume the availability of a few
classical algorithms for graphs: $\textsc{AllTopologicalOrderings}(G)$ to return a
set containing all possible topological orderings of a graph
$G$~\cite{knuth1974structured}, $\textsc{ExistsPath}(G, u, v)$ returns a boolean
value to denote if there is a path from the node $u$ to the node $v$ in the graph
$G$, and $\textsc{MinimumVertexCover}(G)$ to return an MVC of a graph $G$ by
exhaustive search.

\vspace{-1em}
\paragraph{Edit Distance.}
For our purposes, we use the
\emph{Longest Common Subsequence} (LCS) edit distance, which only allows
deletion or
addition of items in the sequence (it does not allow substitution or transposition).
This edit distance is a good fit for our problem because it mimics the
affordances of the user interface of \tool{}.
Throughout the rest of the paper, we will simply use ``edit distance'' to refer to
the LCS edit distance. We denote the edit distance between two
sequences $S_1$ and $S_2$ as $d(S_1, S_2)$.
Formally defined, given two sequences $S_1, S_2$, the edit distance is the length of
the shortest possible sequence of operations that transforms $S_1$ into $S_2$, where
the operations are:
(1) \textbf{Deletion} of element $s_i$: changes the sequence $s_1, s_2,...s_{i-1},
s_i, s_{i+1}, ... s_n$ to the sequence $s_1, s_2,...s_{i-1},s_{i+1}, ... s_n$.
(2) \textbf{Insertion} of element $t$ after location $i$:  changes the sequence
$s_1, s_2,...s_i, s_{i+1}, ... s_n$ to the sequence $s_1, s_2,...s_i,t,s_{i+1}, ...
s_n$.
We assume the ability to compute the edit distance between two sequences in
quadratic time using the traditional dynamic programming
method~\cite{wagner1974string}.
We also identify a topological ordering $O$
of a graph $G$ with a sequence of nodes so that we can discuss the edit distance
between a topological ordering and another sequence $d(S, O)$.

\subsection{Problem Definition}
A \emph{\tool{} problem} $P = (C, G)$ is
a set of \emph{blocks} $C$ together with a DAG $G =
(V,E)$, which defines the logical structure of the proof. Both the blocks and
the graph are provided by the instructor who writes the question
(see~\cite{poulsen2022proof} for more details on question authoring). The set of
vertices $V$ of
the graph $G$ is a subset of the set of blocks $C$. Blocks which are in
$C$ but \emph{not} in $V$ are blocks which are not in any correct solution, and
we call these \emph{distractors}, a term which we borrow from the literature on
multiple-choice questions. A \emph{submission} $S = s_1, s_2, ...s_n$ is a sequence
of distinct blocks, usually constructed by a student who is
attempting to solve a \tool{} problem. If a submission $S$ is a topological
ordering of the graph $G$, we say that $S$ is a \emph{solution} to the \tool{}
problem $P$.

If a student submits a
submission $S$ to a \tool{} problem $P = (C, G)$, we want to assign partial
credit with the following properties:
(1) students get 100\% only if the submission is a solution
(2) partial credit declines  with the number of edits needed to convert
the submission into a solution
(3) the credit received is guaranteed to be in the range $0-100$.
To satisfy these desirable properties, we assign partial credit as follows:
$
\text{score} = 100 \times \frac{\max(0,|V| - d^*)}{|V|},
$
where $d^*$ is the minimum edit distance from the student submission to some
correct solution of $P$, that is:
$
d^* = \min \{d(S, O) \mid O \in \textsc{AllTopologicalOrderings}(G)\}.
$
This means, for example, that if a student's solution is 2 deletions and 1 insertion
(3 edits) away from a correct
solution, and the correct solution is 10 lines long, the student will receive 70\%.
If the edit distance is greater than the length of the solution, we simply assign
0.

\subsection{Baseline Algorithm}
The most straightforward approach to calculating partial credit as we define it is
to iterate over all topological orderings of $G$ and for each one, calculate the
edit
distance to the student submission $S$. We formalize this approach as Algorithm
\ref{alg:baseline}. While this is effective, this algorithm is computationally
expensive.

\begin{thm}
The time complexity of Algorithm \ref{alg:baseline} is $\bigO(m \cdot n
\cdot n!)$ in the worst case, where $n$ is the size of $G$ and $m$ is the length of
the student submission after distractors are discarded.
\end{thm}
\begin{proof}
The algorithm explicitly enumerates all $\bigO(n!)$ topological orderings of $G$.
For each ordering, the algorithm forms the associated block sequence, and computes
the edit distance to the student submission, requiring $\bigO(m \cdot n)$ time.
\end{proof}

\begin{algorithm}
\begin{algorithmic}[1]
\Input
\Desc{$S$}{The student submission being graded}
\Desc{$P$}{The \tool{} problem written by the instructor}
\EndInput
\Output
\Desc{$$}{The minimum number of edits needed to transform $S$ into a solution}
\EndOutput
\Procedure{GetMinimumEditDistance}{$S=s_1, s_2, ... s_\ell$, $P=(C,G)$}
\newline \emph{Brute force calculation of $d^*$:}
\State \Return $\min \{d(S, O) \mid O \in \textsc{AllTopologicalOrderings}(G)\}$
\EndProcedure
\end{algorithmic}
 \caption{Baseline Algorithm}\label{alg:baseline}
\end{algorithm}
\vskip -2em

\subsection{Optimized (MVC-based) implementation of Edit Distance Algorithm}
We now present a faster algorithm for calculating the \tool{} partial credit,
which reduces the problem to the \emph{minimum vertex cover} (MVC)
problem over a subset of the student's submission.
Rather than iterate over all topological orderings, this algorithm works by
manipulating the student's submission until it becomes a correct solution. In order
to do this, we define a few more terms.
We call a pair of blocks $(s_i, s_j)$ in a submission a
\emph{problematic pair} if line $s_j$ comes before line $s_i$ in the student
submission, but there is a path from $s_i$ to $s_j$ in $G$, meaning that $s_j$ must
come \emph{after} $s_i$ in any correct solution.

We define the \emph{problematic graph} to be the graph where the nodes are the
set of all blocks in a student submission that appear in some problematic pair,
and the edges are the problematic pairs.
We can then use the problematic graph to guide which blocks need to be
deleted from the student submission, and then we know that a simple series of
insertions will give us a topological ordering.
The full approach is shown in Algorithm \ref{alg:mvc}, and the proof of
Theorem~\ref{thm:equivalent} proves that this algorithm is correct.

\begin{algorithm}
\begin{algorithmic}[1]
\Input
\Desc{$S$}{The student submission being graded}
\Desc{$P$}{The \tool{} problem written by the instructor}
\EndInput
\Output
\Desc{$$}{The minimum number of edits needed to transform $S$ into a solution}
\EndOutput
\Procedure{GetMinimumEditDistance}{$S=s_1, s_2, ... s_\ell$, $P=(C,G)$}
\newline \emph{Construct the problematic graph:}
\State $E_0$ $\leftarrow$
	$\{(s_i, s_j) \mid i > j \text{ and } \textsc{ExistsPath}(G, s_i, s_j)\}$
\State $V_0$ $\leftarrow$
	$\{s_i \mid \text{there exists } j \text{ such that } (s_i, s_j) \in E_0 \text{
	or }
	(s_j, s_i) \in E_0
	\}$
\State $ \text{problematicGraph} \leftarrow (V_0, E_0)$ \vskip 1em
\hskip -2em \emph{Find number of insertions and deletions needed:}
\State $\text{mvcSize} \leftarrow |
\textsc{MinimumVertexCover}(\text{problematicGraph})|$ \label{alg:line:mvc}
\State numDistractors $\leftarrow$	$|\{s_i \in S \mid s_i \notin V\}|$
\State deletionsNeeded $\leftarrow$
 numDistractors + mvcSize \label{alg:line:deletions}
\State insertionsNeeded $\leftarrow$ $|V| -$ ($|S| - \text{deletionsNeeded}$)
\State \Return $\text{deletionsNeeded} + \text{insertionsNeeded}$
\EndProcedure
\end{algorithmic}
 \caption{Novel Algorithm using the MVC}\label{alg:mvc}
\end{algorithm}
\vskip -3em

\subsection{Worked example of Algorithm \ref{alg:mvc}}

For further clarity, we will now walk through a full example of executing Algorithm
\ref{alg:mvc} on a student submission.
Take, for example, the submission shown in Figure \ref{fig:example1}. In terms of
the block labels, this submission is $S = 1, 3, 4, 5, 2, 7$.
In this case, block $2$ occurs after blocks $3$, $4$, and $5$, but because
of the structure of the DAG, we know that it must come before all of those lines
in any correct solution. Therefore, the problematic graph in this case is
problematicGraph $= (\{2, 3, 4, 5\}, \{(2, 3), (2, 4), (2, 5)\})$. The minimum
vertex
cover here is $\{2\}$, because that is the smallest set which contains at least one
endpoint of each edge in the graph. Now we know that the number of deletions needed
is 1 + 1 = 2 (vertex cover of size one, plus one distractor line picked, see
Algorithm \ref{alg:mvc} line
\ref{alg:line:deletions}), and the
number of insertions needed is 2 (line 2 must be reinserted in the correct position
after being deleted, and line 6 must be inserted). This gives us a least edit
distance ($d^*$) of 4, and so the partial credit assigned would be
$
\text{score} = 100 \times \frac{\max(0,|V| - d^*)}{|V|} = 100
\times
\frac{6 - 4}{6}
\approx 33\%.
$

\subsection{Proving the Correctness of Algorithm \ref{alg:mvc} }
First we will show that the Algorithm constructs a feasible solution, and then we
will
show that it is minimal.

\begin{lem}[Feasability]
Given a submission $S = s_1, s_2, ...s_\ell$, there exists a sequence of edits
$\mathcal{E}$ from $S$ to some solution of $P$ such that
$|\mathcal{E}|$ is equal to $\textsc{GetMinimumEditDistance}(S, P)$ as computed
by Algorithm~\ref{alg:mvc}.
\label{lem:feasability}
\end{lem}
\vspace{-1em}
\begin{proof}
Given the MVC computed on line~\ref{alg:line:mvc} of Algorithm \ref{alg:mvc}, delete
all blocks
in the MVC from $S$, as well as all distractors in $S$, and call this new submission
$S'$. Now $S'$ is a submission such that it contains no distractors, and its
problematic graph is empty.

Now, for all $i$ where $1 \leq i < \ell$, add the edge $(s_i, s_{i+1})$ to the
graph
$G$, and call this new graph $G'$. Because there are no problematic pairs in
$S'$,
we know that adding these new edges does not introduce any new cycles, so $G'$ is a
DAG.
Now, a topological ordering $O$ of the graph $G'$ will be a topological ordering of
$G$
with the added constraint that all blocks which appeared in the submission $S'$
are still in the same order with respect to one	 another. Then, since there are no
distractors in $S'$, $S'$ will be a subsequence of $O$.
Thus, we can construct $O$ simply by adding blocks to $S$.
The length of this sequence $\mathcal{E}$ is exactly what Algorithm~\ref{alg:mvc}
computes.

\end{proof}

\vspace{-1.5em}
\begin{lem}[Minimality]
Let $E'$ be any edit from the submission $S$ to some correct solution of $P$.
Then the length of $E'$ is greater than or equal to the output of Algorithm
\ref{alg:mvc}.
\label{lem:minimality}
\end{lem}
\begin{proof}
Let $E$ be the edit sequence constructed in Lemma~\ref{lem:feasability}.
We will show that the number of deletions and the number of insertions in $E'$ is
greater than or equal to the number of deletions and insertions in $E$.

If
there is any problematic pair $(s_i, s_j)$ in the student
submission, one of $s_i$ and $s_j$ must be deleted from the submission to reach a
solution. Because there is no
substitution or transposition allowed, and because each block may only occur
once in a sequence, there is no other way for the student
submission to be transformed into some correct solution unless $s_i$ or $s_j$ is
deleted and then re-inserted in a different position.

Therefore, the set of blocks deleted in the edit sequence $E'$ must be a vertex
cover of the problematic graph related to $S$
in $E$ and we delete only the blocks in the minimum vertex cover of the
problematic graph. In both cases, all distractors must be deleted. So, the number
of deletions in $E'$ is greater than or equal to the number of deletions in $E$.

The number of insertions in any edit sequence must be the number of deletions, plus
the difference between the length of the submission and the size of the graph, so
that the final solution will be the correct length. Then since the number of
deletions in $E'$ is at least as many as there are in $E$, the number of insertions
in $E'$ is also at least as many as the number of insertions in $E$.

Combining what we have shown about insertions and deletions, we have that $E'$ is at
least as long of an edit sequence as $E$.

\end{proof}

\vspace{-1.5em}
\begin{thm}
Algorithm \ref{alg:mvc} computes  $d^*$---the minimum edit distance from the
submission $S$
to some topological ordering of $G$---in $\bigO(m^2 \cdot 2^m)$ time,
where $m$ is the length of the student submission after distractors are
discarded.
\label{thm:equivalent}
\end{thm}
\begin{proof}
The correctness of the algorithm is given by combining Lemma \ref{lem:feasability}
and Lemma \ref{lem:minimality}.
To see the time complexity, consider that constructing the problematic graph
requires using a breadth first search
to check for the existence of a path between each block and all of the blocks
which precede it in the submission $S$, which can be completed in polynomial
time. Na\"ively computing the MVC of the problematic graph has
time complexity $\bigO(m^2 \cdot 2^m)$. Asymptotically, the calculation of the MVC
will dominate the calculation of the problematic graph, giving an overall time
complexity of $\bigO(m^2 \cdot 2^m)$.
\end{proof}

\begin{remark}
The complexity of Algorithm \ref{alg:mvc} is due to the brute force computation of
the MVC, however, there exists a a $\bigO(1.2738^k + kn)$-time \emph{fixed
parameter
tractable} (FPT) algorithm where $k$ is the size of the minimum vertex cover
\cite{Chen-Kanj-Xia-2010}. While we focus in this paper on the brute force MVC
method since it is sufficient for our needs, using the FPT method may give further
speedup, especially considering the often small size of $k$ in real use cases
(see Table~\ref{tab:comparison}).
\end{remark}

\tool{} also supports a feature known as \emph{block groups} which enables \tool{}
to handle proof by cases using an extended version of Algorithm~\ref{alg:mvc}.
Full details of the extended version of Algorithm~\ref{alg:mvc} and proofs of its
correctness can be seen in the first author's dissertation~\cite{poulsen2023proof}.
In the benchmarking results shown below in Table~\ref{tab:comparison}, problems 2,
4, 6, 9, and 10 make use of this extended version.

\section{Benchmarking Algorithms on Student Data}\label{sec:Benchmarking-Algorithms-on-Student-Data}

\subsection{Data Collection}
We collected data from homework, practice tests, and exams from Discrete
Mathematics courses taught in the computer science departments at the University
of Illinois Urbana-Champaign and the University of Chicago. Problems covered
topics including number theory, cardinality,
functions, graphs, algorithm analysis, and probability theory.
Some questions only appeared on optional practice exams, while others appeared
on exams. Also, more difficult questions received more incorrect submissions as
students were often given multiple tries. This explains the
large discrepancy between the number of submissions to certain questions seen in
Table~\ref{tab:comparison}. In total, our benchmark set includes 7,427 submissions
to 42 different \tool{} problems.

\subsection{Benchmarking Details}
All benchmarking was done in Python. For Algorithm \ref{alg:baseline}, we used the
NetworkX library~\cite{hagberg2008exploring} to generate all topological orderings
of $G$
and used the standard dynamic programming algorithm for LCS edit distance to
calculate the edit distance between each submission and each topological ordering.
Our implementation of Algorithm~\ref{alg:mvc} also used NetworkX to store the
graph, and then found the MVC using
the na\"ive method of iterating over all subsets of the graph, starting from the
smallest to the largest, until finding one which is a vertex cover.
Benchmarks were run on an Intel i5-8530U CPU with 16GB of RAM.
The implementation of Algorithm~\ref{alg:mvc} used for benchmarking is the
implementation used in production in PrairieLearn~\cite{orderblocksdocs}.
Runestone Academy uses an
alternate implementation in JavaScript.

\subsection{Results}
Figure~\ref{fig:sizeandorderingsvstime}
\begin{figure*}[t]
\centering
\includegraphics[width=.95\textwidth]{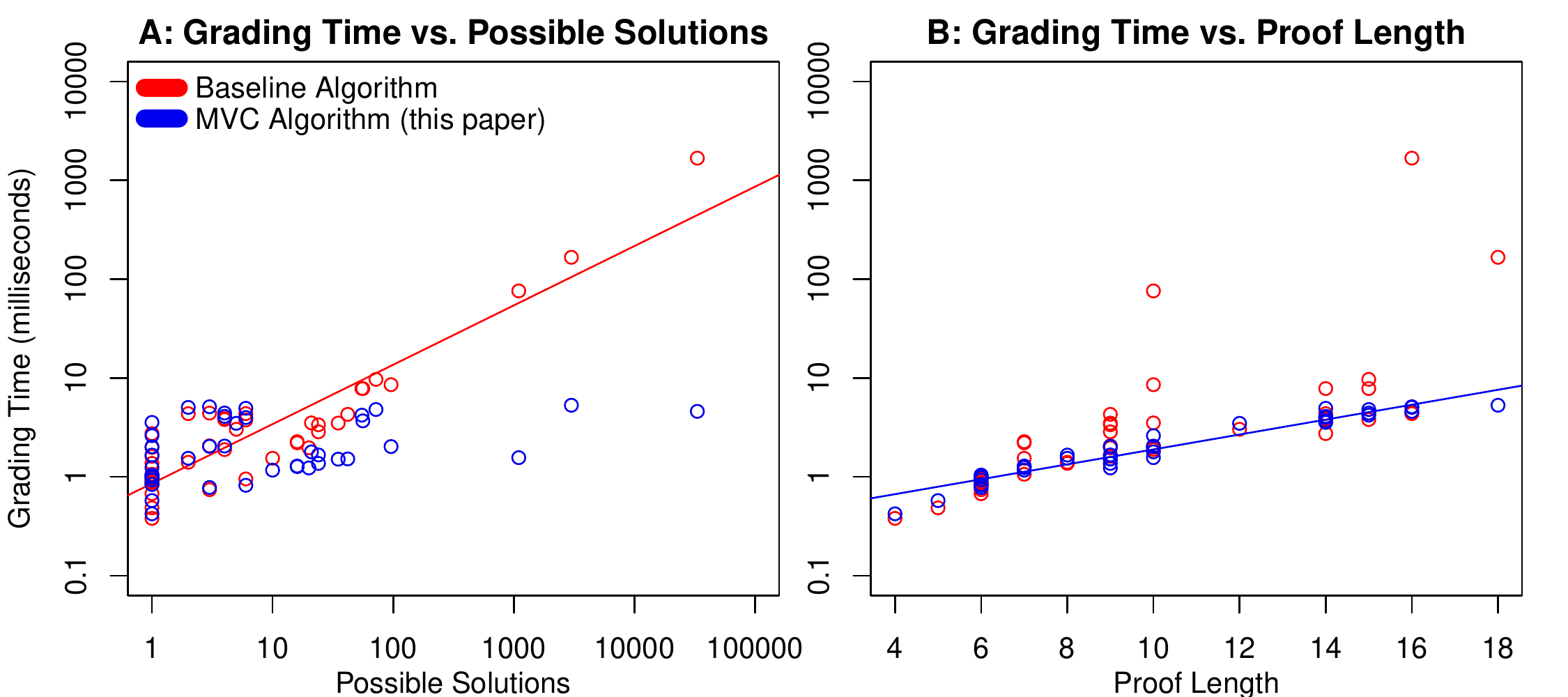}
\caption{
Comparison of grading time for the two grading algorithms.
Subplot (A) is a log-log plot showing that the baseline algorithm
scales with the number of possible solutions,
Subplot (B) is a log-linear plot showing that the MVC Algorithm runtime scales
with the length of the proof. This is a
critical difference, because the number of topological orderings of a DAG can be
$n!$ for a graph with $n$ nodes.
}
\label{fig:sizeandorderingsvstime}
\vspace{-1em}
\end{figure*}
shows the algorithm run time of our
novel MVC-based algorithm (Algorithm~\ref{alg:mvc}) and the baseline algorithm
(Algorithm~\ref{alg:baseline})
for all of the problems in our benchmark set, compared both to the
number of possible solutions (A) and the Proof Length (B).
This demonstrates that the theoretical algorithm run-times hold empirically:
the run time of Algorithm~\ref{alg:baseline}
scales exponentially with the number of topological orderings of the proof graph
(A), and the run time of Algorithm \ref{alg:mvc} scales
exponentially with the length of the proof (B). This is a
critical difference, because the number of topological orderings of a DAG can be
$n!$ for a graph with $n$ nodes.
Thus, a relatively short \tool{} problem could have
a very long grading time with Algorithm~\ref{alg:baseline}, while with
Algorithm~\ref{alg:mvc}, we can guarantee a tight bound on grading time given the
problem
size.

Algorithm 1 performed about twice as fast as Algorithm 2 when there was one
solution---more trivial cases when both Algorithms took less than a
millisecond. Performance was comparable for problems with between 2 and 10
possible solutions
Algorithm 2 performed significantly faster for all problems with over 10 possible
solutions ($p<0.001$).

\setlength{\tabcolsep}{6pt}
\begin{table}[t]
\centering
\begin{tabular}{r|rrr|rrrr|rrr}
  \rot{ Question Number }
& \rot{ Proof Length }
& \rot{ Possible Solutions }
& \rot{ Distractors }
& \rot{ Submissions }
& \rot{ Submission size (mean) }
& \rot{ Prob. Graph Size (mean) }
& \rot{ MVC Size (mean) }
& \rot{ Baseline Alg. Time (mean ms) }
& \rot{ MVC Alg. Time (mean ms) }
& \rot{ Speedup Factor }
\\ \hline
1 & 9 & 24 & 0 & 529 & 8.9 & 3.3 & 1.6 & 3.36 (0.90) & 1.66 (0.50) & 2.0 \\
2 & 9 & 35 & 5 & 376 & 8.5 & 1.2 & 0.4 & 3.50 (0.63) & 1.51 (0.28) & 2.3 \\
3 & 9 & 42 & 0 & 13 & 8.8 & 1.5 & 0.7 & 4.29 (0.40) & 1.52 (0.15) & 2.8 \\
4 & 15 & 55 & 0 & 29 & 6.9 & 2.2 & 0.4 & 7.83 (1.39) & 4.21 (0.75) & 1.9 \\
5 & 14 & 56 & 0 & 324 & 6.4 & 1.8 & 0.6 & 7.82 (1.32) & 3.68 (0.98) & 2.1 \\
6 & 15 & 72 & 0 & 260 & 7.4 & 1.6 & 0.5 & 9.69 (2.20) & 4.80 (1.27) & 2.0 \\
7 & 10 & 96 & 0 & 145 & 8.3 & 2.5 & 1.0 & 8.56 (1.27) & 2.02 (0.48) & 4.2 \\
8 & 10 & 1100 & 0 & 616 & 9.4 & 3.1 & 1.3 & 76.05 (8.78) & 1.56 (0.21) & 48.7 \\
9 & 18 & 3003 & 0 & 253 & 8.2 & 1.6 & 0.5 & 166.4 (15.9) & 5.30 (1.05) & 31.4 \\
10 & 16 & 33264 & 0 & 97 & 4.9 & 1.2 & 0.4 & 1676.0 (235.) & 4.60 (1.43) & 364.8 \\
\hline
\end{tabular}
\caption{
Performance of baseline vs. MVC algorithm for 10 problems with the most
topological sorts. For all problems shown, speedup was statistically significant
at $p<0.001$. Numbers in parentheses are standard errors.}
\label{tab:comparison}
\vspace{-2em}
\end{table}
Table~\ref{tab:comparison} gives further benchmarking details for the 10 questions
from the data set with the greatest number of possible solutions (the others are
omitted due to space constraints).
These results show that Algorithm \ref{alg:mvc} is far superior in performance.
The mean
time of 1.7 seconds for the most complex \tool{} problem under Algorithm
\ref{alg:baseline} may not
seem computationally
expensive, but it does not scale to having hundreds of students working on an
active learning activity, or taking exams at the
same time, all needing to receive rapid feedback on their \tool{} problem
submissions.
Furthermore, this grading time could easily be 10 or even 100 times
longer per question if the DAG for the question was altered by even a single edge.

\section{Conclusions and Future Work}
In this paper, we have presented a novel algorithm for calculating the edit
distance from a student submission to a correct solution to a \tool{} problem.
This information can then be used to give students feedback and calculate grades.
This algorithm can also be used with Parsons Problems, task
planning problems, or any other type of problem where the solution space can be
modeled as a DAG. We showed with student data that
our algorithm far outperforms the baseline algorithm, allowing us to give
students immediate feedback as they work through exams and homework.
Now deployed in dozens of classrooms across multiple universities, this
algorithm benefits thousands of students per semester.

\subsubsection*{Acknowledgements.}
We would like to thank Mahesh Viswanathan, Benjamin Cosman, Patrick Lin, and Tim
Ng for using \tool{} in their Discrete Math courses and allowing us to use
anonymized data from their courses for the included benchmarks.

\bibliographystyle{splncs04}
\bibliography{references}

\end{document}